\documentclass[conference]{IEEEtran}
\IEEEoverridecommandlockouts

\usepackage{multirow}
\usepackage{caption}
\usepackage{subcaption}
\usepackage{amsmath}
\usepackage{cite}
\usepackage{amsmath,amssymb,amsfonts,amsthm,mathptmx}
\usepackage{algorithmic}
\usepackage{graphicx}
\usepackage{textcomp}
\usepackage{xcolor}
\usepackage{epstopdf}
\def\BibTeX{{\rm B\kern-.05em{\sc i\kern-.025em b}\kern-.08em
    T\kern-.1667em\lower.7ex\hbox{E}\kern-.125emX}}

\begin{document}

\title{Improving Generated and Retrieved Knowledge Combination Through Zero-shot Generation\\
}

\newtheorem{theorem}{Theorem}
\newtheorem{lemma}[theorem]{Lemma}
\newtheorem{corollary}[theorem]{Corollary}
\newtheorem{proposition}[theorem]{Proposition}
\newtheorem{fact}[theorem]{Fact}
\newtheorem{definition}{Definition}

\author{
\IEEEauthorblockN{1\textsuperscript{st} Xinkai Du$^{ \dagger}$}
\IEEEauthorblockA{\textit{Dig\&AI Tech} \\
\textit{Sunshine Insurance Group}\\
Beijing, China \\
duxinkai-ghq@sinosig.com}
\and
\IEEEauthorblockN{2\textsuperscript{nd} Quanjie Han$^{ \dagger}$\thanks{$^{\dagger}$Equal contribution.}}
\IEEEauthorblockA{\textit{Dig\&AI Tech} \\
\textit{Sunshine Insurance Group}\\
Beijing, China \\
hanquanjie-ghq@sinosig.com}
\and
\IEEEauthorblockN{3\textsuperscript{rd}  Chao Lv}
\IEEEauthorblockA{\textit{Dig\&AI Tech} \\
\textit{Sunshine Insurance Group }\\
Beijing, China \\
lvchao-ghq@sinosig.com}
\and
\IEEEauthorblockN{4\textsuperscript{th} Yan Liu}
\IEEEauthorblockA{\textit{Dig\&AI Tech} \\
\textit{Sunshine Insurance Group}\\
Beijing, China \\
liuyan02-ghq@sinosig.com}
\and
\IEEEauthorblockN{5\textsuperscript{th} Yalin Sun}
\IEEEauthorblockA{\textit{Dig\&AI Tech} \\
\textit{Sunshine Insurance Group}\\
Beijing, China \\
sunyalin-ghq@sinosig.com}
\and
\IEEEauthorblockN{6\textsuperscript{th} Hao Shu}
\IEEEauthorblockA{\textit{Dept. of Comp. Sci. \& Tech.} \\
\textit{Tsinghua University} \\
Beijing, China \\
13331028785@163.com}
\and
\IEEEauthorblockN{7\textsuperscript{th}  Hongbo Shan}
\IEEEauthorblockA{\textit{Dept. of Comp. Sci. \& Tech.} \\
\textit{Tsinghua University} \\
Beijing, China \\
shanhongbo3203@126.com}
\and
\IEEEauthorblockN{8\textsuperscript{th} Maosong Sun$^{ \star}$\thanks{$^{\star}$Corresponding Author.}}
\IEEEauthorblockA{\textit{Dept. of Comp. Sci. \& Tech.} \\
\textit{Tsinghua University}\\
Beijing, China \\
sms@tsinghua.edu.cn}
}
\maketitle

\begin{abstract}
Open-domain Question Answering (QA) has garnered substantial interest by combining the advantages of faithfully retrieved passages and relevant passages generated through Large Language Models (LLMs). However, there is a lack of definitive labels available to pair these sources of knowledge. In order to address this issue, we propose an unsupervised and simple framework called Bi-Reranking for Merging Generated and Retrieved Knowledge (BRMGR), which utilizes re-ranking methods for both retrieved  passages and LLM-generated passages. We pair the two types of passages using two separate re-ranking methods and then combine them through greedy matching. We demonstrate that BRMGR is equivalent to employing a bipartite matching loss when assigning each retrieved passage with a corresponding LLM-generated passage. The application of our model yielded experimental results from three datasets, improving their performance by +1.7 and +1.6 on NQ and WebQ datasets, respectively, and obtaining comparable result on TriviaQA dataset when compared to competitive baselines.
\end{abstract}

\begin{IEEEkeywords}
unsupervised passage reranking, retrieval augmentation, generation augmentation, open question answering.
\end{IEEEkeywords}

\section{Introduction}
\label{sec:intro}
 
In the realm of knowledge-intensive tasks, such as open-domain question answering, a vast repository of world and domain-specific knowledge is paramount. A commonly employed approach, known as "retrieve-then-read" \cite{lewis2020retrieval,karpukhin2020dense,izacard2021leveraging} involves leveraging external corpora to retrieve relevant passages. Subsequently, a reader model processes the query and retrieved contexts to furnish answers. Nevertheless, this approach is prone to limitations. Specifically, the candidate documents utilized for retrieval are typically chunked into fixed-length segments (e.g., 100 words), leading to the retrieval of noisy or irrelevant information that does not directly pertain to the query at hand \cite{yu2022generate,gao2022precise}.

\begin{figure}[h]
  \centering
  \includegraphics[width=3in,height=2in]{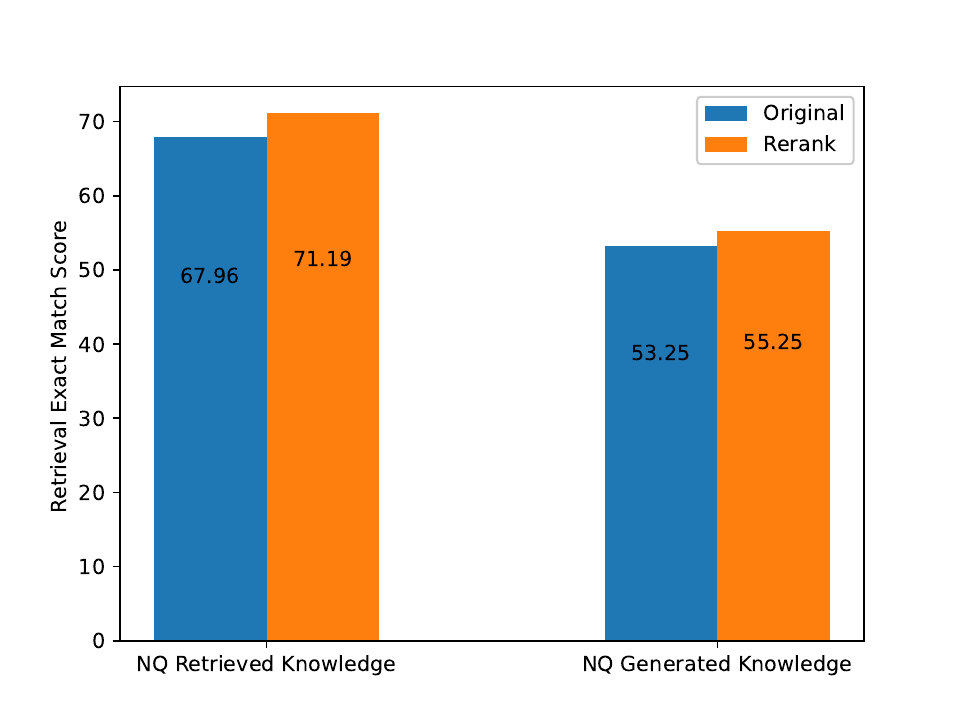}\\
\caption{ Top-3 retrieval exact match score for single knowledge source after reranking knowledge sources.}
\label{fig:introduction_fig}
\end{figure} 

Recently, the strong generative capabilities of Large Language Models (LLMs) \cite{brown2020language,chowdhery2023palm, touvron2023llama} have garnered significant attention. These models trained on vast amounts of unsupervised data, possess the ability to encode a wide range of knowledge within their parameters. LLMs serve as an alternative information source, complementing traditional retrieval methods \cite{brown2020language, singh2021end, zhang2023retrieve}. In fact, an alternative framework known as "generate-then-read" has been proposed, which generates query-related contexts rather than retrieving them from a corpus \cite{yu2022generate, li2023generate}.

Drawing from the foundations established by generation-augmented and retrieval-augmented methods, recent hybrid approaches aim to integrate these techniques to further enhance performance in Question Answering (QA) tasks \cite{mallen2022not,zhang2023merging,abdallah2023generator}. Recognizing that retrieval-augmented methods may introduce unrelated content while generation-augmented methods may yield relevant yet plausible contexts, a compatibility-oriented framework has been introduced to merge retrieved knowledge and LLM-generated knowledge \cite{zhang2023merging}. This framework strives to maintain the factuality of retrieved knowledge while leveraging the relevance of LLM-generated knowledge. To assess compatibility between the two knowledge sources, discriminators such as the evidentiality discriminator and the consistency discriminator have been proposed. However, these methods often rely on silver label mining \cite{asai2022evidentiality,zhang2023relevance,zhang2023merging}, which can be complex and data-intensive.

Motivated by the limitations of existing approaches and inspired by unsupervised reranking methods for retrieval-augmented QA \cite{ponte2017language,sachan2022improving,zhuang2023open,santos2020beyond}, we find that reranking both the retrieved knowledge and generated knowledge \cite{tan2024blinded} can improve the performance (Figure \ref{fig:introduction_fig}). Furthermore, we introduce an unsupervised method for computing compatibility scores between retrieved and LLM-generated knowledge. This approach leverages zero-shot generation to determine compatibility, enabling a simpler and more efficient reranking process. Our method not only combines the strengths of retrieved and generated knowledge but also offers a complementary reranking approach for generated knowledge. Extensive experimental results demonstrate the effectiveness of our method, paving the way for future research in this exciting domain.


\section{Proposed Method}
An overview of our unsupervised Bi-Reranking for Merging Generated and Retrieved knowledge (BRMGR) method is depicted in Figure \ref{BRMGR}. It illustrates the process of combining the faithfulness of retrieved knowledge with the relevant evidence from LLM-generated knowledge using unsupervised learning.

\begin{figure*}[h]
\centerline{\includegraphics[height=3.5in]{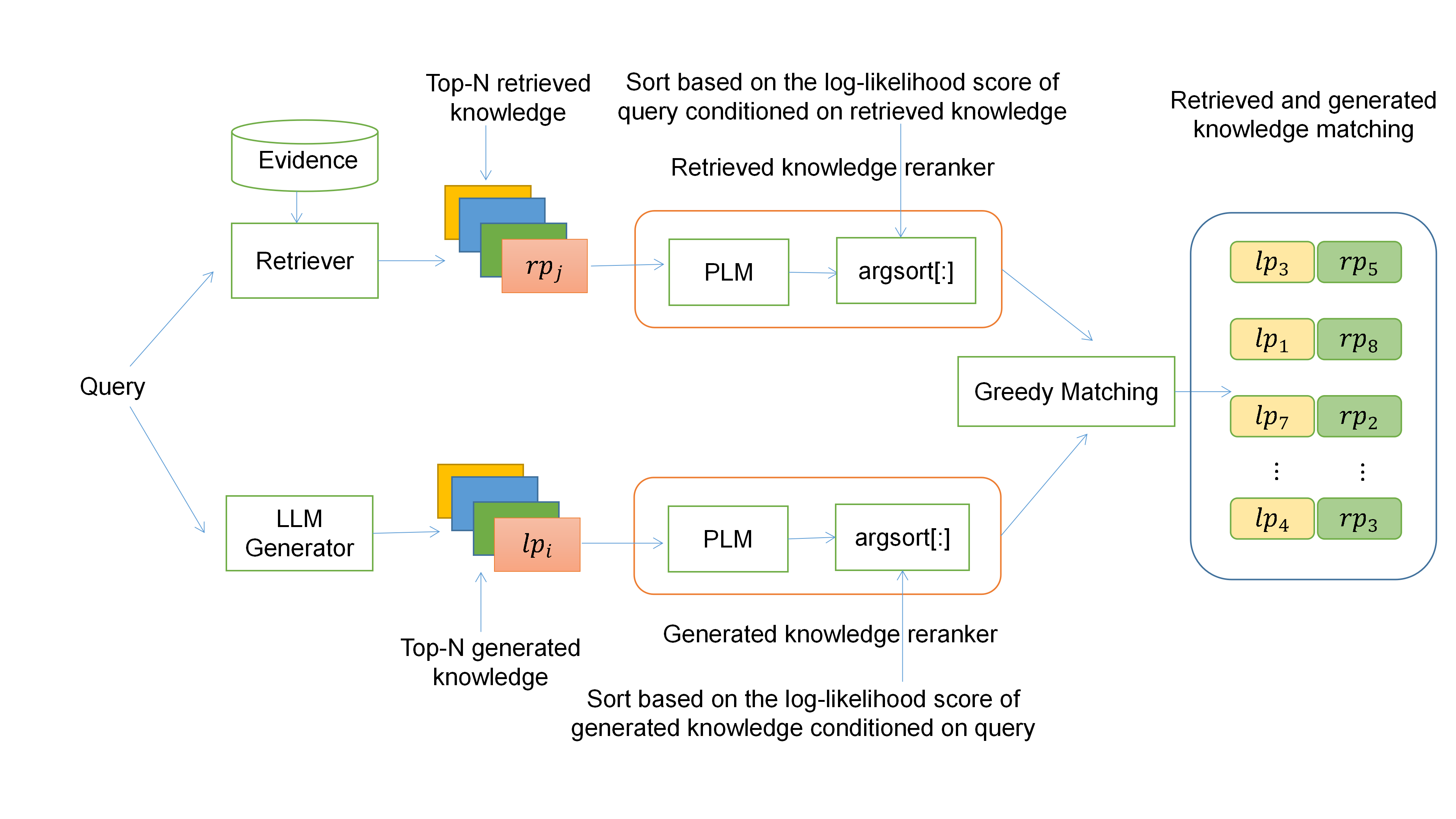}}
\caption{Overview of the BRMGR Framework: It uses an unsupervised method to rerank both LLM-generated and retrieved knowledge for Open-Domain QA. We compute the relevance score of retrieved knowledge based on the log-likelihood score of the query conditioned on retrieved knowledge, and compute the relevance score of generated knowledge via the log-likelihood score of  generated knowledge conditioned on query. Finally, the retrieved knowledge and generated knowledge are combined using a greedy matching approach. } 
\label{BRMGR}
\end{figure*}

Let $\textbf{P}_L=\{\textbf{lp}_1,\textbf{lp}_2,...,\textbf{lp}_M\}$, $\textbf{P}_R=\{\textbf{rp}_1,\textbf{rp}_2,...,\textbf{rp}_N\}$ be a collection of the LLM-generated passages and retrieved passages respectively. Given a query $\textbf{q}$, the goal is to find an optimal combination between $\textbf{lp}_i$ and $\textbf{rp}_j$ such that the correct answer is ranked as highly as possible. The combination relevance score is computed with a language model $p(\textbf{lp}_i,\textbf{rp}_j|\textbf{q}), i=1,2,...,M, j=1,2,...,N$.

Given the query $\textbf{q}$, assume the combination relevance score between $\textbf{lp}_i$ and $\textbf{rp}_j$ is conditional independent, then it can be factorized into two probabilities:
\begin{equation}
p(\textbf{lp}_i,\textbf{rp}_j|\textbf{q})=p(\textbf{lp}_i|\textbf{q})p(\textbf{rp}_j|\textbf{q})
\label{probability_factorization}
\end{equation}
In general, the retrieved passages may contain content unrelated to the query. The Unsupervised Passages Reranking (UPR) suggests the use of the zero-shot query likelihood method to improve passage retrieval \cite{sachan2022improving}.
\begin{equation}
p(\textbf{rp}_j|\textbf{q})=\frac{p(\textbf{q}|\textbf{rp}_j)p(\textbf{rp}_j)}{p(\textbf{q})}\propto p(\textbf{q}|\textbf{rp}_j)p(\textbf{rp}_j)
\label{retrieval_upr}
\end{equation}

Assume $p(\textbf{rp}_j)$ is uniform, then equation (\ref{retrieval_upr}) reduces to 

\begin{equation}
p(\textbf{rp}_j|\textbf{q})\propto p(\textbf{q}|\textbf{rp}_j), \forall \textbf{rp}_j \in \textbf{P}_R
\label{upr_method_explaination}
\end{equation}

Similar to UPR \cite{sachan2022improving}, the combination relevance score in retrieval component is defined as: 
\begin{equation}
\log p(\textbf{q}|\textbf{rp}_j)=\frac{1}{|\textbf{q}|}\sum_t\log p(q_t|\textbf{q}_{<t},\textbf{rp}_j;\Theta)
\label{retrieved_upr_method}
\end{equation}
where $\Theta$ denotes the parameters of a pretrained language model and $|\textbf{q}|$ denotes the number of query tokens.

 Due to the Hallucination of large language model \cite{tonmoy2024comprehensive}, the generated passages may contain some unhelpful content. Given the large language model's strong ability, the generated information is highly related to the query. We propose to utilize the conditional probability of the generated document conditioned on query $p(\textbf{lp}_i|\textbf{q})$ to rerank the generated passages. 

\begin{equation}
p(\textbf{lp}_i|\textbf{q})=\frac{1}{|\textbf{lp}_i|}\sum_t\log p({lp}_{i,t}|\textbf{lp}_{i,<t},\textbf{q};\Theta)
\label{generated_upr_method}
\end{equation}
where $\Theta$ denotes the parameters of a pretrained language model and $|\textbf{lp}_i|$ denotes the number of tokens for the $i$-th generated passage.

Finally the relevance of how the $i$-th generated passage and the $j$-th retrieved passage related to the query is computed by:
\begin{equation}
p(\textbf{lp}_i,\textbf{rp}_j|\textbf{q})\propto p(\textbf{lp}_i|\textbf{q})p(\textbf{q}|\textbf{rp}_j)
\label{proposed_method}
\end{equation}
where $p(\textbf{lp}_i|\textbf{q})$ and $p(\textbf{q}|\textbf{rp}_j)$ are computed by equation (\ref{generated_upr_method}) and (\ref{retrieved_upr_method}), respectively.

Since the combination of computed relevance scores of retrieved and LLM-generated knowledge results in a $M\times N$ matrix, it raises an interesting question: why not use the bipartite matching loss \cite{sui2023joint,du2024label} and Hungarian algorithm \cite{kuhn1955hungarian} to find an optimal match?


\begin{theorem} 
If we assume that the combination relevance score of $p(\textbf{lp}_i,\textbf{rp}_j|\textbf{q})$ can be factorized into the relevance scores of the generated knowledge and the retrieved knowledge, and further, if the number of both types of knowledge is the same, then we can conclude that the optimal match obtained through bipartite matching loss is equivalent to the one obtained through the greedy matching.
\end{theorem}

\begin{proof}
 We will prove this statement using mathematical induction. Since the combination relevance score of $p(\textbf{lp}_i,\textbf{rp}_j|\textbf{q})$
can be factorized, we define $a_{ij}=p(\textbf{lp}_i,\textbf{rp}_j|\textbf{q})=b_ic_j, i,j=1,2,...,N$. 

For $k=1$, this is clearly true. Assuming that it holds true for $k=N-1$, we can show that it also holds true for $k=N$ using the Hungarian algorithm. 

The maximum value of 
$a_{ij}$ is the product 
of the maximum value of $b_i$ and  $c_j$, thus the top-1 combination is the combination of the top-1 generated knowledge and the top-1 retrieved knowledge. By induction, we can solve the remaining $N-1$ pairs using a greedy matching.

\end{proof}

\section{Experiments}
\subsection{Experimental Setup}
\subsubsection{Datasets} Following previous work \cite{zhang2023merging}, we use the three popular open-domain Question Answering (QA) datasets of TriviaQA \cite{joshi2017triviaqa},
Natural Questions (NQ) \cite{kwiatkowski2019natural} and WebQuestions (WebQ) \cite{berant2013semantic}. 

For the generated knowledge, we use those provided by \cite{yu2022generate} which were produced by prompting InstructGPT \cite{ouyang2022training}. There are 20 generated passages for each question with human prompts and 10 generated passages for clustering-based prompt method.

\subsubsection{Evaluation Metric} To evaluate the effectiveness of the proposed method, we measure its performance using the conventional \emph{top-K retrieval exact match} metric \cite{rajpurkar2016squad}. This metric calculates the proportion of questions for which at least one passage within the top-K passages contains a span that matches the human-annotated answer for that question \cite{sachan2022improving}.
\subsubsection{Baselines}
Comparable to \cite{zhang2023merging}, we employ single and two knowledge sources in comparison experiments to demonstrate our method's effectiveness.

For single knowledge sources, we consider four methods:
\begin{itemize}
\item \textbf{Retri-Origin}: Utilizes retrieved knowledge directly.
\item \textbf{Retri-Rerank}: Reranks retrieved knowledge using UPR.
\item \textbf{Gen-Origin}: Utilizes generated knowledge directly.
\item \textbf{Gen-Rerank}: Reranks generated knowledge using UPR.
\end{itemize}
When dealing with two knowledge sources (retrieved and generated), we propose two additional methods:
\begin{itemize}
\item \textbf{Origin-Combi}: Simply merges the original retrieved and generated knowledge without reranking.
\item \textbf{COMBO}: COMBO matches LLM-generated passages with their retrieved counterparts to form compatible pairs. It's solely utilized as a comparative technique for our proposed approach in open question answering scenarios.
\item \textbf{BRMGR}: Our proposed unsupervised Bi-Reranking for Merging Generated and Retrieved knowledge method, which aims to optimize the combination of both sources.

\end{itemize}

\subsubsection{Implementation Details}
Following the approach in \cite{sachan2022improving}, we employ a range of Flan-T5 models \cite{raffel2020exploring}, varying from base to xlarge versions. Additionally, we consider the T0-3B model \cite{victor2022multitask} due to its promising performance in retrieved passages \cite{sachan2022improving}. To rerank the retrieved passages, the verbalizer is set as \emph{Please write a question based on this passage}. To provide context, the title and text of each passage are concatenated with a verbalizer head as follows: \textrm{"passage: "}. All experiments are conducted on a single 80G A100 GPU.

\subsection{Main Results}
All methods in this study utilize 10 retrieved and/or LLM-generated passages for each question. The set of 10 retrieved passages is selected based on the top-10 passages returned by the Dense Passage Retrieval (DPR) method \cite{karpukhin2020dense}. On the other hand, the 10 generated passages are sampled from the passages generated by humans prompts in \cite{yu2022generate}.

\textbf{Retrieval} The results of the \emph{top-K retrieval exact match} metric, using the T0-3B model, are presented in Table \ref{recall_em_result} for the three datasets. It is evident that our method achieves the best performance, and the utilization of two knowledge sources outperforms using a single knowledge source. Notably, for the TriviaQA and WebQ test sets, the generated passages are found to be more helpful compared to the retrieved passages. Additionally, following the implementation of two separate unsupervised re-ranking methods, the performance of both the retrieved and generated passages exhibits improvement.

\begin{table}
\begin{center}
{\caption{Top-\{3, 5\} exact match result on test set before and after reranking of the 10 retrieved and/or generated passages. }\label{recall_em_result}}
\begin{tabular}{lcccccc}
\hline
\multirow{2}{*}{Methods }&\multicolumn{2}{c}{\textbf{TriviaQA }}&\multicolumn{2}{c}{\textbf{NQ}}&\multicolumn{2}{c}{\textbf{WebQ}}\\
 &Top-3&Top-5 &Top-3&Top-5 &Top-3&Top-5  \\
 \hline
\textit{\scriptsize Single Knowledge}&\multicolumn{6}{c}{}\\
Retri-Origin&68.57&72.40&63.35&68.78&60.09&65.01\\
Retri-Rerank&73.27&74.48&65.48&70.97&61.61&66.29\\
Gen-Origin&75.71&78.66&54.63&59.92&62.11&66.58\\
Gen-Rerank&75.78&78.82&56.62&61.58&64.76&68.31\\
\textit{\scriptsize Two Knowledge}&\multicolumn{6}{c}{}\\
Origin-Combi&82.34&84.41&75.68&79.78&74.61&78.20\\
BRMGR&\textbf{83.06}&\textbf{84.81}&\textbf{77.01}&\textbf{81.44}&\textbf{74.90}&\textbf{78.49}\\
 \hline
\end{tabular}
\end{center}
\end{table}

\textbf{Question Answering} In order to assess the performance of our method in open question answering, we employ Fusion-in-Decoder \cite{izacard2021leveraging} as the reader model, with COMBO \cite{zhang2023merging} serving as the baseline method. The results of this evaluation are presented in Table \ref{oqa_em_result}.

The Table \ref{oqa_em_result} clearly demonstrates that our method outperforms the baseline approach, achieving the strongest overall performance. In addition, re-ranking the original order of both the retrieved and generated passages leads to improved results. Moreover, incorporating two knowledge sources for open question answering proves to be more effective than relying on a single knowledge source alone.

\begin{table}
\begin{center}
{\caption{Exact match scores computed by FiD on test dataset.}\label{oqa_em_result}}
\begin{tabular}{p{2cm}p{1.5cm}p{1.5cm}p{1.5cm}}
\hline
\rule{0pt}{12pt}
Methods&\textbf{TriviaQA}&\textbf{NQ}&\textbf{WebQ}\\
 \hline
\textit{\scriptsize Single Knowledge}&\multicolumn{3}{c}{}\\
Retri-Origin&66.3&50.8&50.1\\
Gen-Origin&69.6&52.6&42.6\\
\textit{\scriptsize Two Knowledge}&\multicolumn{3}{c}{}\\
COMBO &\textbf{74.6}&54.2&53.0\\
Origin-Combi&73.2&53.8&52.4\\
Retri-Rerank&73.5&54.4&53.1\\
Gen-Rerank&73.8&54.6&53.3\\
BRMGR&74.4&\textbf{55.9}&\textbf{54.6}\\
 \hline
\end{tabular}
\end{center}
\end{table}

\subsection{Ablation Studies}
\subsubsection{Effect of the Pretrained-Language Models}
To examine the impact of pretrained language models (PLMs) on the exact match score of generated passages, a range of Flan-T5 models \cite{raffel2020exploring} is utilized. Additionally, T0-3B model \cite{victor2022multitask} is utilized to evaluate the performance on the NQ development set. Results presented in Table \ref{GRF_upr_result} demonstrate that all the pretrained language models contribute to overall performance improvement in generated passages. This contrasts with the size-dependent phenomenon observed in retrieved passages. Since the generated passages are generated by these powerful and large language models, reranking with different PLM sizes yields similar results.
\begin{table}
\begin{center}
{\caption{Comparison of different pre-trained language models (PLMs) as re-rankers for the generated passages on the NQ development set.}\label{GRF_upr_result}}
\begin{tabular}{lcccc}
\hline
\rule{0pt}{12pt}
\multirow{2}{*}{Generated Reranker} &\multicolumn{4}{c}{\textbf{NQ (dev)}}\\
& Top-1&Top-3&Top-5 &Top-8 \\
\hline
\\[-8pt]
\quad None&39.34&53.25&58.22&62.49\\
\quad T5-base&\textbf{40.11}&55.24&59.63&\textbf{62.96}\\
\quad T5-large&40.09&55.28&59.62&\textbf{62.96}\\
\quad T5-xlarge&40.09&\textbf{55.29}&59.61&62.92\\
\quad T0-3B&40.07&55.25&\textbf{59.68}&62.94\\
\hline
\end{tabular}
\end{center}
\end{table}
\subsubsection{Importance of Document Generation}
To grasp the significance of re-ranking using document generation $p(\textbf{lp}|\textbf{q})$, we contrast it with an alternative unsupervised method where re-ranking is question generation conditioned on the generated knowledge $p(\textbf{q}|\textbf{lp})$. This value can be approximated by calculating the average log-likelihood of generating the question tokens using PLM and teacher-forcing.
\begin{equation}
\log p(\textbf{q}|\textbf{lp})=\frac{1}{|\textbf{q}|}\sum_t\log p(q_t|\textbf{q}_{<t},\textbf{lp};\Theta)
\label{ablation_generated_upr_method}
\end{equation}

From the observed results in figure \ref{test_generated_UPR}, it is evident that the reranking scores, computed by the log-likelihood of generating passage tokens conditioned on the query, consistently enhance the original outcome. However, the reranking scores computed by the log-likelihood of generating the query conditioned on the generated knowledge demonstrate a deterioration in performance. Notably, this performance degradation becomes more significant with a smaller size of the PLM.

\begin{figure}[h]
\centerline{\includegraphics[height=2.5in]{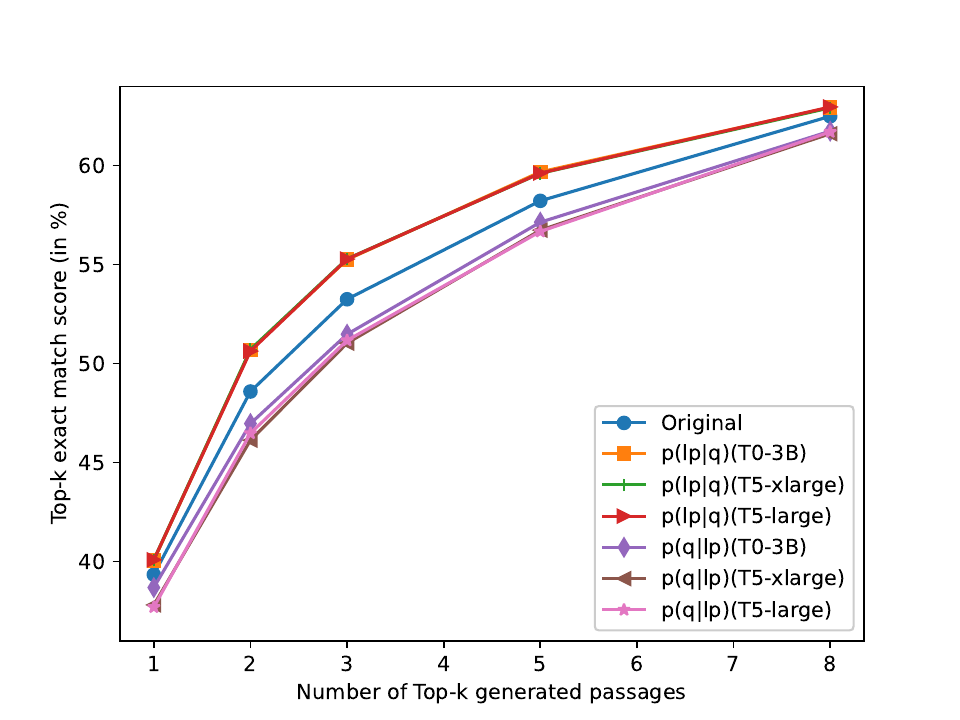}}
\captionsetup{justification=justified}
\caption{Comparison of two passage re-ranking approaches on the NQ development set: (1) when generating question tokens conditioned on the passage $p(\textbf{q}|\textbf{lp})$, and (2) when generating passage tokens conditioned on the question $p(\textbf{lp}|\textbf{q})$. Results highlight
the usefulness of document generation in generated knowledge for reranking.} \label{test_generated_UPR}
\end{figure}
\section{Conclusion}
In this work, we propose an unsupervised bi-reranking method BRMGR for merging retrieved passages and LLM-generated passages in Open-domain Question Answering. Rather than relying on mined silver labels for computing compatibility scores between the two types of passages. Extensive experiments on three datasets demonstrate the success of this proposed method.

\textbf{Acknowledgements}
This work is supported by Beijing Municipal Science and Technology Plan Project (Z241100001324025).

\vfill\pagebreak



\bibliographystyle{IEEEtran}
\bibliography{refs}

\end{document}